\documentclass{llncs}
\usepackage{latexsym}
\usepackage{amsmath}
\usepackage{amssymb}

\newenvironment{Proof}{\medbreak \noindent {\bf
Proof:~}}{\unskip\nobreak\hfill\hskip 2em \bull \par\medbreak}
\def\bull{\vrule height .9ex width .8ex depth -.1ex }
\numberwithin{equation}{section}
\newcommand{\E}{\ensuremath{\mathbb{E}}}

\newcommand{\abs}[1]{\left| #1 \right|}
\newcommand{\bt}{\tilde{b}}
\newtheorem{Def}[theorem]{Definition}
\newtheorem{observation}[theorem]{Observation}

\newcommand{\R}{\ensuremath{\mathbb{R}}}
\usepackage{graphicx} 

\usepackage{algorithm}
\usepackage{algorithmic}

\usepackage[backref]{hyperref}

\begin{document} 

\title{Optimal amortized regret in every interval}
\titlerunning{Optimal amortized regret in every interval}  

\parindent 1em

\author{Rina Panigrahy\inst{1}
\and
Preyas Popat\inst{2}
}

\institute{
Microsoft Corp. \\ Mountain View, CA \\ rina@microsoft.com
\and
New York University and University of Chicago \\ New York, NY and Chicago, IL \\ popat@cs.nyu.edu
}

\maketitle

\begin{abstract} 
Consider the classical problem of predicting the next bit in a sequence of bits. 
A standard performance measure is {\em regret} (loss in payoff) with respect to a set of experts. 
For example if we measure performance with respect to two constant experts one that always predicts $0$'s and another that always predicts $1$'s it is well known that one can get regret $O(\sqrt T)$ 
with respect to the best expert by using, say, the weighted majority algorithm \cite{weighted-majority}. But this algorithm does not provide performance guarantee in any interval. 
There are other algorithms (see \cite{blum-mansour,freund-schapire-singer-warmuth,vovk}) that ensure regret $O(\sqrt {x \log T})$ in any interval of length $x$. In this paper we show
a randomized algorithm that in  an amortized sense gets a regret of $O(\sqrt x)$ for any interval when the sequence is partitioned into intervals arbitrarily. 
We empirically estimated the constant in the $O()$ for $T$ upto $2000$ and found it to be  small -- around $2.1$. We also experimentally evaluate the efficacy of this algorithm in predicting high frequency stock data.

\end{abstract}

\section{Introduction}
Consider the following classical game of predicting a binary $\pm 1$
sequence.  An algorithm $A$ sees a binary sequence $\{b_t\}_{t\ge
  1}$, one bit at a time, and attempts to predict the next bit $b_t$
from the past history $b_1,\ldots b_{t-1}$. The {\em payoff} $A_T$ of the algorithm in $T$ steps
is the number of correct guesses minus the number
of the wrong guesses.  In other words, let $\bt_t \in [-1,1]$ be the prediction for the $t^{th}$ bit based on the previous bits then:
$$ 
A_T := \sum_{1\le t\le T} b_t\bt_t.
$$

The payoff per time step $b_t\bt_t$ is essentially equivalent to the well known
absolute loss function $|b_t-\bt_t|$ (see for example ~\cite{plg}, chapter 8).\footnote{since when $|b_t| = 1$, $|b_t-\bt_t| = |b_t||b_t-\bt_t| = |1 - b_t \bt_t| = 1 - b_t \bt_t$.  Thus the absolute loss function is the negative of our payoff in one step plus a shift of $1$.  Also $b_t$ values from $\{-1,1\}$ or $\{0,1\}$ are equivalent by a simple scaling and shifting transform.}

One can view this game as an idealized ``stock prediction'' problem as
follows. In each unit time, the stock price goes up or down by precisely $1\%$, and the algorithm bets on this event. If the bet is right, the
player wins one dollar, and otherwise loses one dollar.  Not
surprisingly, in general, it is impossible to guarantee a positive
payoff for all possible scenarios (sequences). 
However, one could hope to give some guarantees on the payoff
of the algorithm based on certain properties of the sequence.

For example one can compare the payoff to the better of two choices (experts),
which correspond to two constant algorithms: first one, where $\bt_t=+1$
and the second one where
$\bt_t=-1$ for all $t$. Note that the best of these experts gets payoff 
$|\sum_{1\le t\le T}  b_t|$, which
corresponds to the ``optimal in hindsight'' expert among the two choices. 
The {\em regret} of an algorithm is defined as how much
worse the algorithm performs as opposed to the best of the two
experts (in hindsight, after seeing the sequence). This has been
studied in a number of papers, including \cite{cover-binary,weighted-majority,cover-portfolios,auer-nonstoch,audibert-bubeck}. A classical result says that one can obtain a
regret of $\Theta(\sqrt{T})$ for a sequence of length $T$, via, say,
the weighted majority algorithm \cite{weighted-majority}.  Formally, for a sequence $X = b_1, \ldots, b_T$, let $h(X)=\sum_{1\le t\le T}  b_t$  denote the ``height" of the sequence when plotted cumulatively as a chart.    
Then we have the following theorem:
\begin{theorem}\cite{cover-binary,cesa1997use}
There is an algorithm that achieves payoff $\ge |h(X)| - \alpha\sqrt{T}$. It is also known that 
the optimal value of $\alpha \rightarrow \sqrt{2/\pi}$ as $T \rightarrow \infty$.
\end{theorem}


However, an algorithm that only focuses on the overall regret does not exploit short term trends in the sequence and only relies on a `global' long term bias in the full string. Consider for example a sequence
that may not have a high overall bias but has many intervals in which there may be a high level of bias. Our result is that for any partitioning of the
sequence into intervals, one can essentially get a regret proportional to $\sqrt x$ for each interval of length $x$ in an amortized sense (Theorem \ref{main}).  Although our results are stated for bits they work even when $b_t$ is a real number in $[-1,1]$.   We note that even though similar bounds  have been obtained before (\cite{blum-mansour,freund-schapire-singer-warmuth,vovk} and, more recently, \cite{hazan-seshadhri,KP11}), the penalty on an interval of length $x$ is $O(\sqrt {x \log T})$ in these previous results.  

The bit prediction problem we consider is closely related to the two experts
problem (or multi-armed bandits problem with full information).
In each round each expert has a
payoff in the range $[0,1]$ that is unknown to the algorithm. For two
experts, let $b_{1t}, b_{2t}$ denote the payoffs of the two
experts at time $t$. The algorithm pulls each arm (expert) with probability
$\bt_{1t}, \bt_{2t} \in [0,1]$ respectively where $\bt_{1t} + \bt_{2t}
= 1$.  The payoff of the algorithm in this setting is $A'_T := \sum_{t=1}^T b_{1t}\bt_{1t}
+ b_{2t}\bt_{2t} $. 

We will be concerned with the following payoff function in this paper:

\begin{Def} \label{interval-payoff} (Interval payoff function: $P_\alpha$)

Let $X_1, \ldots, X_k$ denote a partition of the sequence $X$ into a disjoint union of $k$ intervals. that is, $X$ is the concatenation of these $k$ subsequences. We will use
$h(X_i)$ to denote the sum of the bits in the interval $X_i$ and $\abs{X_i}$ to denote the  length of $X_i$.

The interval payoff function, $P_\alpha(X)$ is defined as the maximum value of the expression
$$ \displaystyle\sum_{i = 1}^k \left( |h(X_i)| -  \alpha \sqrt{|X_i|} \right) $$
over all $1 \leq k \leq |X|$ and all partitions $X_1, \ldots, X_k$ of $X$.
\end{Def}

We say that a payoff function $f : \{-1, 1\}^T \to \R$ is feasible if there is a bit prediction algorithm which on sequence $X$ achieves payoff at least $f(X)$.

\begin{theorem} {\bf(Main Theorem)} \label{main}
There is an absolute constant $\alpha < 10$ such that the payoff function $P_\alpha$ is feasible.
\end{theorem}

For the two experts problem our result tranlates to the following guarantee:
 $$ A'_T \geq \sum_{i = 1}^k \left(  \max_{j \in {1,2}} \left(\sum_{t \in X_i}  b_{jt} \right) -  \frac{\alpha}{2} \sqrt{|X_i|} \right) .$$
Here $\sum_{t \in X_i}  b_{jt}$ is the payoff of the $j^{th}$ expert in the interval $X_i$.

This can be viewed as incurring a penalty of $\alpha \sqrt{|X_i|}$ for each interval  $X_i$.   We theoretically show that the optimal value of $\alpha$ is at most $10$ (Section \ref{proof}). We empirically estimated the optimal $\alpha$ for $T$ up to $2000$ and found it to be  small -- around $2.1$ (Section \ref{alpha}).  

We stress here that the algorithm doesn't need to know the partition or the length of the partition in advance.  We also note that our guarantee does not hold for each interval individually but when we look at the net payoff in an amortized sense, we may account for a regret of at most $\alpha \sqrt{|X|}$ for an interval of length $X$.  In fact, the guarantee is impossible to achieve in a non-amortized sense.  We show that if we measure regret based on the performance of an algorithm in a given interval then one will have to trade-off regrets at different time scales.

\begin{observation} (Observation \ref{impossible})
There is no prediction algorithm that can guarantee a regret of $O( \sqrt{|Y|})$ on all intervals $Y$ for all  input sequences.
\end{observation}

Regarding the computation of $P_\alpha$, we show:

\begin{theorem} \label{compute} (Theorem \ref{algorithm})
The value of $P_\alpha(S)$ for a particular sequence $S$ of length $T$ can be computed using dynamic programming in time $O(T^3)$.
\end{theorem}

For a given $T$, let $\alpha_0(T)$ denote the minimum $\alpha$ such that $P_\alpha$ is feasible for all sequences of length $T$.  It is possible to determine $\alpha_0$ using the following well known observation
by Cover.

\begin{observation}[Cover~\cite{cover-binary}] \label{cover} 
A payoff function $f : \{-1, 1\}^T \to \R$  is feasible if and only if $E_S[f(S)] \le 0$ where $S$ is a uniformly random sequence in $\{-1, 1\}^T$.

This is achieved by a prediction algorithm that predicts
$\bt_t =  \frac{E_{U}[f(s.1.U)] - E_{U}[f(s.(-1).U)]}{2}$ where $s$ is the sequence of bits seen so far, $U$ is a suffix sequence chosen uniformly at random and $s.b.U$ denotes the concatenated sequence starting with $s$ followed by bit $b$ followed by the sequence $U$.
Note that $\bt_t \in [-1,1]$ as long as for all  $s$, $|E_{U}[f(s.1.U)] - E_{U}[f(s.(-1).U)]| \le 2$ 
\end{observation}

{\bf Algorithm and Running time:}  
Theorem \ref{compute} and Observation \ref{cover} suggest a simple algorithm for achieving payoff function $P_\alpha$.  Take the sequence $s$ seen so far, append a $+1$ and then a random sequence to make it into a complete sequence of length $T$. Compute $P_\alpha(S)$ for the resulting sequence $S$. Do this again replacing the $+1$ by a $-1$. Predict $\bt_t$ to be the half of the difference in the two cases.

We note that a deterministic algorithm achieving the guarantee of Theorem \ref{main} may take exponential time since it would need to find $P_\alpha(S)$ for every random completion of the bits seen so far.
Alternatively, there is a simple randomized algorithm which achieves the same payoff in expectation by taking a different random completion for every prefix.   A naive implementation of this randomized algorithm 
will take $T^3$ time for each bit being predicted. We show a simple variant that reduces this to $O(\log T)$ time with
pre-computation.

\begin{theorem} (Theorem \ref{Compute})
There is a randomized algorithm that achieves the payoff guarantee $P_\alpha$ of Theorem \ref{main} in expectation and spends $O(T^2)$ time per step.
There is also a randomized algorithm that achieves payoff $P_{\alpha'}$ with $\alpha' = c \alpha$ and spends only
$O(\log T)$ time per step.  Here $c := \frac{\sqrt{2}}{\sqrt{2} - 1}$. 

Both algorithms above use pre-computed information that takes $O(T^2)$ space and is computed in $O(T^4)$ time.
\end{theorem}

{\bf Generalization to real numbers:}  We show that a variant of the guarantee holds in a semi-adversarial model where a string of real numbers may be chosen
instead of bits. The model combines worst case and average case settings where the signs of the real numbers may be chosen adversarially (that is, in the worst case)
but the magnitudes of the real numbers come from a pre-specified distribution independently and randomly (Theorem \ref{real}) .

{\bf Experimental results: }  We implement our algorithm, the weighted majority algorithm,  an algorithm based on Autoregressive Integrated Moving Average (ARIMA) and an algorithm of \cite{KP11}, and compare their performance when
predicting financial time series data.  Specifically, we consider the high frequency price data of $5$ stocks, and we apply these algorithms to predict the per minute price changes in an online fashion taking the values in each day as a separate 
sequence. That is we predict the next minute returns of mid-prices for each stock based on its previous $1$ minute returns in the day.  We perform this experiment over $189$ trading days for each stock and find
that on an average our algorithm performs better than other prediction algorithms based on regret minimization but is outperformed by the ARIMA algorithm.  On the other hand, as we discussed above, our algorithm has certain provable guarantees
for \emph{every} sequence which the ARIMA algorithm lacks.  The experimental setup and results are described in more detail in Section \ref{experiment}.

\subsection{Related work}

There is  large body on work on regret style analysis for prediction. Numerous
works including~\cite{cover-binary,cesa1997use} have examined the optimal amount
of regret achievable with respect to two or more experts. A good reference for the results in this area is~\cite{plg}.  It is well known that in the case of
static experts, the optimal regret achievable is exactly equal to the
Rademacher complexity of the predictions of the experts (chapter 8 in~\cite{plg}).
Recent works such as~\cite{abernethy2006continuous,abernethy2008optimal,mukherjee2008learning} have extended this analysis to other settings.
Measures other than the standard regret measure have been studied in~\cite{rakhlin2010online}  The question of what can be achieved if one would like to have a significantly better guarantee with respect to a fixed expert or a distribution of experts was asked before in \cite{kearns-regret,KP11}.  Tradeoffs between regret and minimum payoff were also examined in \cite{vovk-game}, where the author studied the set of values of $a, b$ for which an algorithm can have payoff $a OPT+b\log N$, where $OPT$ is the payoff of the best arm and $a, b$ are constants.  

Regret minimization algorithms with performance guarantees within each interval have been studied in \cite{blum-mansour,freund-schapire-singer-warmuth,vovk} and more recently in \cite{hazan-seshadhri,KP11}.
As we mentioned, some of these algorithms achieve a regret of $O(\sqrt {x \log T})$ for every interval of size $x$ in a sequence of length $T$.  A related work which also seeks to exploit short term trends in the sequence is \cite{HW}, where the regret bound proportional to $\sqrt{T k}$ in the best case where $k$ is the number of intervals (see \cite{plg}, Corollary 5.1).  The main difference between the work of \cite{HW} and our results is that their algorithm
requires fixing the number of intervals, $k$, in advance whereas our algorithm works simultaneously for all $k$. Also note
that their regret guarantee is always higher than the payoff function $P_\alpha$ for a sequence of length $T$ achieving equality only in the special case when all intervals are of equal length $T/k$.

Numerous papers (for example~\cite{blum1997empirical,helmbold1998line,agarwal2006algorithms}) have implemented algorithms inspired from regret style analysis and applied it on financial and other types of data.

\subsection{Overview of the proof}

In this section we give a high level idea of our proof, the formal proof appears in Section \ref{proof}.

To prove the main theorem we want to compute the minimum $\alpha$ such that $E_S [P_\alpha(S)] \leq 0$ (See Observation \ref{cover}). We first introduce a variant of the payoff function $P_\alpha(S)$ as follows.  Instead of computing the maximum value of $\sum_i |h(X_i)| -  \alpha \sqrt{|X_i|}$
over all possible partitions, will only allow partitions where the intervals are of the form $(2^i j, 2^i (j+1)]$; that is, intervals that are obtained by dividing the string into
segments of length that are some power of $2$. We will refer to such intervals as `aligned' intervals (Definition \ref{align-def}). Further we will only look at $T$ values that is some power of $2$. Note that any interval can be broken into at most $\log T$
aligned intervals. Let $P^A_\alpha(S)$ denote the maximum value of  $\sum_i |h(X_i)| -  \alpha \sqrt{|X_i|}$ with partitions into aligned intervals.  We first show that

\begin{lemma} (Lemma \ref{align-proof}) \label{align-lem}
If $E[P^A_{ \alpha}(S)]  \le 0$ then $E[P_{c \alpha}(S)]  \le 0$ where $c := \frac{\sqrt{2}}{\sqrt{2} - 1}$.
\end{lemma}

Next we show

\begin{theorem} \label{main-part-2} (Theorem \ref{align-proof})
There is an absolute constant $\alpha \leq 2.8$ such that $E[P^A_\alpha(S)]  \le 0$.
\end{theorem}

We prove Theorem \ref{main-part-2} recursively for $T$ that are increasing powers of $2$. 
We inductively show that the distribution of $P^A_\alpha(S)$ is stochastically upper bounded by a shifted exponential distribution (Definition \ref{shifted-exp}) with certain parameters (Equation \ref{dominance}), where $S$ is a uniformly random sequence of length $T$.
Since we are dealing with splits into aligned intervals, we can assume that either the best split for $S$ is the whole interval, or the mid-point of $S$ is one of the splitting points.  For the first case, we may
upper bound the payoff function using Hoeffding's bound (Theorem \ref{hoeffding}), while for the second case we may inductively assume that the distribution of payoffs for the subsequences is stochastically
bounded by a shifted exponential distribution.  We then separately bound each of this distributions by the shifted exponential distribution.

\section{Proof of Main theorem} \label{proof}

\subsection{Preliminaries}

\begin{Def} (Binomial distribution $B_n$)
Let $x_1, x_2, \ldots, x_n \in \{-1, 1\}$ be uniformly and independently distributed.  Then the sum 

$$ Y := \sum_{i=1}^n x_i $$

is said to be binomially distributed.  We denote the distribution as $B_n$.

\end{Def}

\begin{theorem} (Hoeffding's bound) \cite{hoeffding1963probability} \label{hoeffding}
$$ \Pr[|B_n| \geq y \cdot \sqrt{n}] \leq 2 \cdot \exp\left( - \frac{y^2}{2} \right) $$
\end{theorem}

\begin{Def} (Aligned interval) \label{align-def}

We assume here that $T$ is a power of $2$.  An aligned interval is one which is obtained by breaking $[1, T]$ into $2^i$ equal parts for $i \in [0, \log T]$ and picking one of the parts.
So for instance the first part is always $[1, 2^i]$.

In other words, an interval $[p+1, p+x]$ given by $p \in [0, T]$, $x \in [1, T- p]$ as discussed above is said to be an aligned interval if $p = j \cdot 2^i$ and $x = 2^i$ for some
$i \in [0, \log T]$ and $j \in [0, T - 2^i]$.

\end{Def}

We denote the interval payoff function corresponding to Definition \ref{interval-payoff} which allows only aligned splits as $P_\alpha^A$.  

\begin{Def} \label{shifted-exp} (Shifted Exponential distribution)
The probability density function $f_{\mu, \sigma, n}$ of shifted exponential
distribution with mean $\sigma \sqrt{n}$ and shift $\mu \sqrt{n}$ is defined as follows:
\begin{align*}
f_{\mu, \sigma, n}(y) & := \frac{1}{\sigma \sqrt{n}} \exp\left(- \frac{y - \mu \sqrt{n}}{\sigma \sqrt{n}} \right)   & \forall y \geq \mu \sqrt{n} \\
f_{\mu, \sigma, n}(y) & := 0  & \forall y \leq \mu \sqrt{n}
\end{align*}

We denote a random variable distributed according to $f_{\mu, \sigma, n}$ as $F_{\mu, \sigma, n}$.  That is, $\Pr [F_{\mu, \sigma, n} \geq y]  = \int_{y}^\infty f_n(s) \, \mathrm{d}s  = \exp\left(- \frac{y - \mu \sqrt{n}}{\sigma \sqrt{n}} \right)$ when 
$y \geq \mu \sqrt{n}$ and $1$ otherwise.

\end{Def}

\subsection{Proof}

\begin{theorem}  \label{align-proof}
There is an absolute constant $\alpha \leq 2.8$ s.t. there is an algorithm which achieves payoff greater than $P^A_\alpha$ for all $T \geq 1$.
\end{theorem}
\begin{Proof}
We need to show that for all $T \geq 1$, $\E_{x \in \{-1,1\}^T} [P^A_\alpha(x)] \leq 0$.  After that, the theorem follows from Observation \ref{cover} (it is easy to check that the second condition
of Observation \ref{cover} is satisfied for $P^A_\alpha$).

We will prove the theorem by induction.  We will show that when $n$ is a power of $2$, 
\begin{equation} \label{dominance}
\forall y \in \R \quad  \Pr_{x \in \{-1,1\}^n} [P^A_\alpha(x) \geq y] \ \ \leq \ \  \Pr [F_{\mu, \sigma, n} \geq y]
\end{equation} 

for some $\mu := \mu(\alpha)$ and $\sigma := \sigma(\alpha)$.  Here $F_{\mu, \sigma, n}$ is as in Definition \ref{shifted-exp}.

Note that this would imply $\E_{x \in \{-1,1\}^n} [P^A_\alpha(x)] \leq \E[F_{\mu, \sigma, n}] = (\mu + \sigma) \sqrt{n}$.  We will show that for a suitable choice of $\alpha$,
the term $\mu + \sigma \leq 0$, and this suffices to prove the theorem.

It remains to prove Equation \ref{dominance}.  For the base case, $n=1$, we see that the equation is satisfied for $\mu \geq 1 - \alpha$, $\sigma > 0$.  We will now show that it is satisfied for $2n$ whenever
it is satisfied for $n$ (for appropriate $\mu$ and $\sigma$).

Now, for a sequence $x := (x_1, x_2) \in \{-1,1\}^{n} \times \{-1,1\}^{n}$, $P_\alpha^A(x) = \max(P_\alpha^A(x_1) + P_\alpha^A(x_2), \abs{h(x)} - \alpha \cdot \sqrt{2n})$.  So for every $x$ such that
$P_\alpha^A(x) \geq y$ we must have either $P_\alpha^A(x_1) + P_\alpha^A(x_2) \geq y$ or that $h(x) - \alpha\cdot\sqrt{2n} \geq y$.  Thus,
\begin{align} 
& \Pr_{x \in \{-1,1\}^{2n}} [P_\alpha^A(x) \geq y]  \\
\leq & \Pr_{x_1, x_2 \in \{-1,1\}^{n}} [P_\alpha^A(x_1) + P_\alpha^A(x_2) \geq y] +  \Pr_{x \in \{-1,1\}^{2n}} [h(x) - \alpha\cdot\sqrt{2n} \geq y] \\
\leq & \Pr [F_{\mu, \sigma, n} + F'_{\mu, \sigma, n} \geq y]  +  \Pr_{x \in \{-1,1\}^{2n}} [h(x) - \alpha\cdot\sqrt{2n} \geq y] \label{two-terms}
\end{align}
Here $F$ and $F'$ are independent random variables distributed as in Definition \ref{shifted-exp}.  We will show that the first and second term are each bounded by $\frac{1}{2} \Pr[F_{2n} \geq y]$ which is sufficient to
prove Equation \ref{dominance}.  Note that we only need to consider $y \geq \mu \sqrt{2n}$ since for smaller values of $y$ we have 
$$ \Pr_{x \in \{-1,1\}^{2n}} [P_\alpha^A(x) \geq y] \leq  \Pr[F_{2n} \geq y] = 1 $$

Henceforth, we will use shorthands $f_n := f_{\mu, \sigma, n}$ and $F_n := F_{\mu, \sigma, n}$.

The first term can be written as:-
\begin{align*}
\Pr [F_n + F'_n \geq y] = & \int_{y}^\infty \int_{-\infty}^{\infty} f_n(s) \cdot f_n(w - s) \, \mathrm{d}s \, \mathrm{d}w \\
= & \int_{y}^\infty \int_{\mu \sqrt{n}}^{w - \mu \sqrt{n}} f_n(s) \cdot f_n(w - s) \, \mathrm{d}s \, \mathrm{d}w
\end{align*}
where  the second equation follows from the fact that $f_n(s) = 0$ for $s < \mu \sqrt{n}$ and $f_n(w - s) = 0$ for $s > w - \mu \sqrt{n}$.  
Thus, we need to show for all $ y \geq \mu \sqrt{2n}$:-
\begin{align*}
& \int_{y}^\infty \int_{\mu \sqrt{n}}^{w - \mu \sqrt{n}} f_n(s) \cdot f_n(w - s) \, \mathrm{d}s \, \mathrm{d}w  \leq \frac{1}{2} \Pr[F_{2n} \geq y] \\
\Longleftarrow & \frac{1}{\sigma^2 n} \int_{y}^\infty \int_{\mu \sqrt{n}}^{w - \mu \sqrt{n}} \exp\left(- \frac{s - \mu \sqrt{n}}{\sigma \sqrt{n}} \right) \cdot \exp\left(- \frac{w - s - \mu \sqrt{n}}{\sigma \sqrt{n}} \right) \, \mathrm{d}s \, \mathrm{d}w  \\
& \leq \frac{1}{2} \exp\left(- \frac{y - \mu \sqrt{2n}}{\sigma \sqrt{2n}} \right) \\
\Longleftarrow & \frac{1}{\sigma^2 n} \int_{y}^\infty \int_{\mu \sqrt{n}}^{w - \mu \sqrt{n}} \exp\left(- \frac{w - 2 \mu \sqrt{n}}{\sigma \sqrt{n}} \right) \, \mathrm{d}s  \, \mathrm{d}w \\
 & \leq \frac{1}{2} \exp\left(- \frac{y - \mu \sqrt{2n}}{\sigma \sqrt{2n}} \right) \\
\Longleftarrow & \frac{1}{\sigma^2 n} \int_{y}^\infty (w - 2 \mu \sqrt{n}) \exp\left(- \frac{w - 2 \mu \sqrt{n}}{\sigma \sqrt{n}} \right)  \, \mathrm{d}w \\
 & \leq \frac{1}{2} \exp\left(- \frac{y - \mu \sqrt{2n}}{\sigma \sqrt{2n}} \right) \\
\end{align*}
In the third line we implicitly assume that $y \geq 2 \mu \sqrt{n}$, since otherwise the left hand side is less than $0$ and the equation is satisfied.

Note that the integral is of the form $\int u \cdot e^{-cu}$ which integrates to $- \left( \frac{u + 1/c}{c} \right) \cdot e^{-cu} $.  Thus, integrating and 
substituting $z := y - 2 \mu \sqrt{n}$ we need to show for all $z \geq 0$,
\begin{align*} 
 & \frac{1}{\sigma \sqrt{n}} \cdot (z + \sigma \sqrt{n}) \cdot \exp\left(- \frac{z}{\sigma \sqrt{n}} \right) &  \leq & \frac{1}{2} \exp\left(- \frac{z + (\sqrt{2} - 1) \mu \sqrt{2n}}{\sigma \sqrt{2n}} \right) \\
 \Longleftarrow & \frac{2z}{\sigma \sqrt{n}} + 2 & \leq & \exp\left(\frac{z}{\sigma \sqrt{n}} - \frac{z + (\sqrt{2} - 1) \mu \sqrt{2n}}{\sigma \sqrt{2n}} \right) \\
 \Longleftarrow & \frac{2z}{\sigma \sqrt{n}} + 2 & \leq & \exp\left(\frac{(\sqrt{2} - 1) z}{\sigma \sqrt{2n}} \right) \cdot \exp \left( (\sqrt{2} - 1) \frac{ - \mu}{\sigma } \right) \\
\end{align*} 
Substituting $w := \frac{z}{\sigma \sqrt{n}}$, we need for all $w \geq 0$,
\begin{align*}
& 2w + 2  & \leq & \exp\left(\frac{(\sqrt{2} - 1) w}{\sqrt{2}} \right) \cdot \exp \left( (\sqrt{2} - 1) \frac{ - \mu}{\sigma } \right) \\
 \Longleftarrow & \frac{2w + 2}{\exp\left(\frac{(\sqrt{2} - 1) w}{\sqrt{2}} \right)} & \leq & \exp \left( (\sqrt{2} - 1) \frac{ - \mu}{\sigma } \right) \\
\end{align*}
The left hand side is maximized at $w = 1/\sqrt{2}$ and the value of left hand side at that point is around $2.78$.  Thus, if $(-\mu/\sigma) \geq 2.47$ then the equation is always
satisfied.

We now turn to bounding the second term in Equation \ref{two-terms}.  We need to show for all $y \geq \mu \sqrt{2n}$,
\begin{align*}
& \Pr_{x \in \{-1,1\}^{2n}} [\abs{x} - \alpha\cdot\sqrt{2n} \geq y] \leq \frac{1}{2} \Pr[F_{2n} \geq y] \\
\Longleftarrow & \Pr[|B_{2n}| \geq y + \alpha \cdot \sqrt{2n}] \leq \frac{1}{2} \Pr[F_{2n} \geq y] \\
\Longleftarrow & \Pr[|B_{2n}| \geq (z + \alpha) \cdot \sqrt{2n}] \leq \frac{1}{2} \Pr[F_{2n} \geq z \cdot \sqrt{2n}] \\
\Longleftarrow & 2 \cdot \exp\left( - \frac{(z + \alpha)^2}{2} \right) \leq \frac{1}{2} \Pr[F_{2n} \geq z \cdot \sqrt{2n}] \\
\end{align*}
where the last line follows from Theorem \ref{hoeffding}, and in the second last line we substitute  $z := y/\sqrt{2n}$.


Thus, we need to show for all $z \geq \mu$,

\begin{align*}
& 4 \cdot \exp\left( - \frac{(z + \alpha)^2}{2} \right) & \leq \exp\left(- \frac{z \sqrt{2n} - \mu \sqrt{2n}}{\sigma \sqrt{2n}} \right) \\
\end{align*}

Substituting $w := z - \mu$, we need to show for all $w \geq 0$,

\begin{align*}
 &  \exp\left(- \frac{(w + \mu + \alpha)^2}{2} + \frac{w}{\sigma} \right)  & \leq 0.25 \\
 \Longleftarrow &  - \frac{(w + \mu + \alpha)^2}{2} + \frac{w}{\sigma}   & \leq -1.4 \\
\end{align*}

The left hand side is maximized at $w + \mu + \alpha = 1/\sigma$ and for that value of $w$ the inequality is given by

\begin{align*}
 & \frac{-1}{2 \sigma^2} + \frac{1/\sigma - \mu - \alpha}{\sigma} \leq -1.4 \Longleftarrow \mu + \alpha \geq 1.4 \sigma + \frac{0.5}{\sigma}
\end{align*}

Also, recall that to bound the first term we needed $-\frac{\mu}{\alpha} \geq 2.47$.  Let's set $\mu := -2.47 \sigma$.  Then we need 

$$ \alpha \geq (1.4 + 2.47) \sigma + \frac{0.5}{\sigma} = 3.87 \sigma + \frac{0.5}{\sigma} $$

The right hand side is minimized at $\sigma = \frac{1}{\sqrt{2 \cdot 3.87}} \approx 0.36$, and substituting we get that $\alpha = 2.8$ is feasible.  Recall that we also needed 
$\mu + \alpha \geq 1$ from the base case which is already satisfied for this choice of parameters. 

\end{Proof}

\section{Algorithm and running time}

\begin{theorem} \label{algorithm}
The value of $P_\alpha(S)$ for a sequence $S$ of length $T$ can be computed by a dynamic program (DP) in time $O(T^3)$.
\end{theorem}
\begin{proof}
We give a simple $O(T^2)$ space and $O(T^3)$ time algorithm.  

For every subinterval $(i,j)$ of the sequence, $i, j \in [T]$ the DP table stores $P_\alpha(S_{ij})$ where $S_{ij}$ is the subsequence of $S$
containing bits from position $i$ to position $j$, inclusive.  For $i=j$, this value is always $1 - \alpha$.  For $j > i$, to compute the value of $P_\alpha(S_{ij})$, we need to take the maximum over two quantities.  The first quantity is $\abs{h(S_{ij})} - \alpha \cdot \sqrt{j - i + 1}$ which
corresponds to splitting the subsequence into a single interval.  This can be readily computed in constant time if we pre-compute the height of every subsequence, which can be done in
$O(T^2)$ space and time.  The second quantity is the maximum over all $k \in \{i, i+1, \ldots, j\}$ of $P_\alpha(S_{ik}) + P_\alpha(S_{kj})$.  This corresponds to splitting the subsequence at $k$ and then recursively
computing the best payoff in each of the two intervals created.
This quantity can be computed in time $j - i + 1$ since for each $k$ we just need to read off the appropriate values ($P_\alpha(S_{ik})$ and $P_\alpha(S_{kj})$) from the DP table.  
\end{proof}

\begin{theorem} \label{Compute}
There is a randomized algorithm that achieves the payoff guarantee $P_\alpha$ of the main theorem in expectation and spends $O(T^2)$ time per step.
There is also a randomized algorithm that achieves payoff $P_{\alpha'}$ with $\alpha' = c \alpha$ and spends only
$O(\log T)$ time per step.  Here $c := \frac{\sqrt{2}}{\sqrt{2} - 1}$. 

Both algorithms above use pre-computed information that takes $O(T^2)$ space and is computed in $O(T^4)$ time.
\end{theorem}
\begin{proof}

Let $X \in \{-1, 1\}^T$ be the input sequence we are required to predict.  Using Observation \ref{cover}, it is easy to see that the following algorithm achieves payoff $P_\alpha(X)$ in expectation.  For every $t \in \{0, 1, \ldots, T-1\}$:
\begin{enumerate}
\item Let $s \in \{-1,1\}^t$ be the sequence of bits seen so far.
\item Let $U_t$ be a sequence drawn uniformly at random from $\{-1,1\}^{T - t - 1}$ (independently for each $t$).  Let $s_1 := s \cdot 1 \cdot U$ and $s_{-1} := s \cdot (-1) \cdot U$.
\item Make the prediction $\bt := (P_\alpha(s_1) - P_\alpha(s_{-1})/2$ for the next bit.
\end{enumerate}

The key idea is that we will draw the random sequences $U_t$ in advance and pre-compute enough information to make the prediction as fast as possible.  For each $t \in \{0, 1, \ldots, T-1\}$ we pre-compute the following information for each $U_t$:-
\begin{enumerate}
\item $h(U^1_t)$ for every prefix $U^1_t$ of $U_t$
\item $P_\alpha(U^2_t)$ for every suffix $U^2_t$ of $U_t$
\end{enumerate}

The pre-computation takes $O(T^3)$ time for each $t$ and hence $O(T^4)$ time overall.

Let's describe how to use this pre-computed information to compute $P_\alpha(s_1)$ at time $t$ (the computation of $P_\alpha(s_{-1})$ is similar).  Let $1 \leq i \leq t$ and $t + 2 \leq j \leq T$.  Then it is easy to check
that 

$$ P_\alpha = \max_{i, j} \left[ P_\alpha(s_{1i}) + P_\alpha(U_{jT}) + \abs{h(s_{(i+1)t})} + \abs{h(U_{(t+1)(j-1)})} - \alpha \cdot \sqrt{j - i - 1} \right] $$

Here for a sequence $S$, $S_{ij}$ is the subsequence of $S$ containing bits from position $i$ to position $j$, inclusive.  Note that we think of $U_t$ as being indexed from $t+1$ to $T$ where the $(t+1)^{th}$ bit is $1$ (since we are dealing with $s_1$).  The second and fourth term are part of our pre-computation.  The first and third terms can be computed on the fly and stored in the table as we increase $t$ from $1$ to $T$.  Thus, for each $i$ and $j$ we can compute this expression in constant time and hence we can produce a prediction in $O(T^2)$ time per step.

The second part of the theorem is proved in a similar manner by using only aligned intervals for splitting the sequence (Definition \ref{align-def}) and observing that the number of aligned intervals spanning
a given position is at most $O(\log T)$.  The algorithm achieves payoff at least $P_{\alpha'}$ because of Lemma \ref{align-lem}.

\end{proof}

\bibliography{regret}
\bibliographystyle{splncs}

\appendix

\section{Experimental results} \label{experiment}

In this section we describe our experimental setup and findings.  

The first part of the experiment is to experimentally estimate the value of $\alpha_0$.  In general we may think of $\alpha_0$ as a function of $T$.  In Section \ref{proof} we saw that 
$\alpha_0(T)$ is bounded from above by an absolute constant for all $T$.  In Section \ref{alpha} below we estimate the values of $\alpha_0$ for a range of $T$.

The second part of the experiment is to implement our algorithm and compare its performance against $3$ other prediction algorithms.  This is described in Section \ref{predict} below.

\subsection{Computation of $\alpha_0$} \label{alpha}

We denote by $\alpha_0(T)$ the minimum value of $\alpha$ such that the payoff function $P_\alpha$ is feasible for sequences of length $T$.  For a particular $T$, this value can be computed using Theorem \ref{compute}.  While Theorem \ref{compute} requires us to compute the payoff function over all sequences of length $T$ (to compute the expectation), we can
experimentally approximate this by taking sufficiently many random sequences of length $T$ and looking at the expectation of the sample.  We are interested in $T=389$ which is the number of minutes in a trading day for which we have returns data
(there are $390$ minutes in a typical trading day and the returns for the first minute is undefined).  

%

Note that the standard error of the sample mean is obtained as the sample standard deviation divided by $\sqrt{n}$ where $n = 400$ is the number of trials.  The following chart shows the mean payoff and standard error for various values of
$\alpha$ for $T = 389$.

\includegraphics[scale=0.45]{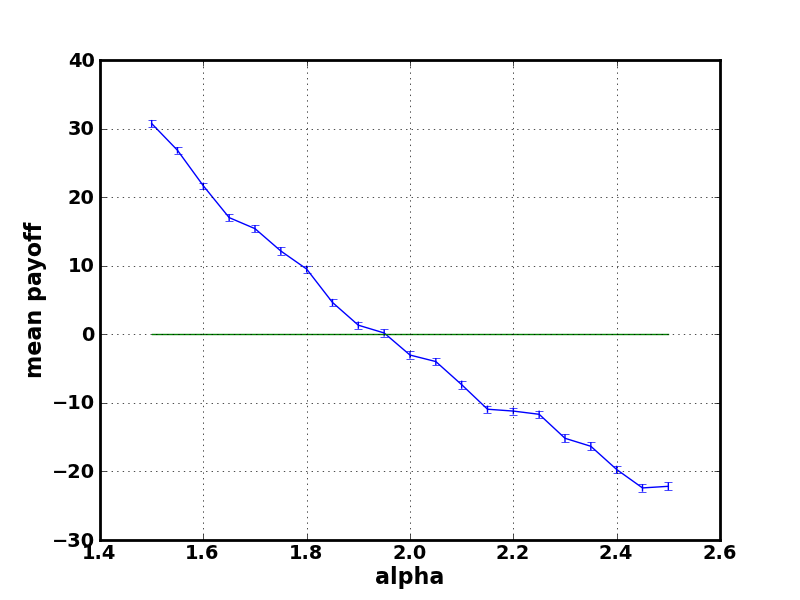}

From the figure we see that $\alpha = 1.96$ is a good estimate for $\alpha_0 (T)$ for $T = 389$.  The figure below shows estimated values of $\alpha_0$ for various $T$.

\includegraphics[scale=0.45]{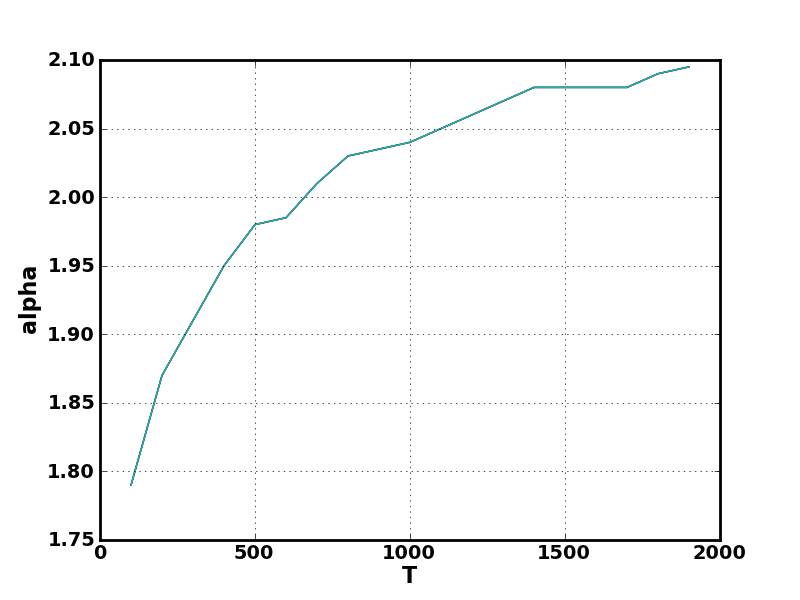}

\subsection{Comparison of predictive performance} \label{predict}

The algorithms we consider are:-
\begin{enumerate}
\item The baseline buy and hold strategy that achieves payoff equal to the height (height)
\item The algorithm described in this paper (interval)
\item Weighted Majority algorithm (WM)
\item The algorithm of \cite{KP11} (Algorithm 4, section 5) (boundedloss)
\item An algorithm based on Auto Regressive Integrated Moving Average (arima)

\end{enumerate}
Note that algorithms $2$-$4$ are based on ideas from regret minimization with provable guarantees while the fifth is a commonly used model for predicting time series data.  To implement the fourth algorithm we
use the function {\sc auto.arima()} in {\sc R} which is part of the library {\sc forecast}.

The prediction task we consider is to predict the next minute returns for a stock over a single trading day using only the previous $1$ minute returns of the given stock for the given day.  More precisely, we define the 
price of a stock at a given time taking the average of the best bid price and best ask price at that time as reported by the New York Stock Exchange (NYSE).
We perform this prediction experiment over
$189$ days for the following $5$ US stocks/ETFs from various sectors: {\sc MSFT}, {\sc GE}, {\sc GLD}, {\sc QQQ} and {\sc WMT}.  This gives us performance data for each algorithm for a total of $389 \times 189 \times 5 = 367,605$ data points.
The results obtained are shown in the figure below.

\includegraphics[scale=0.35]{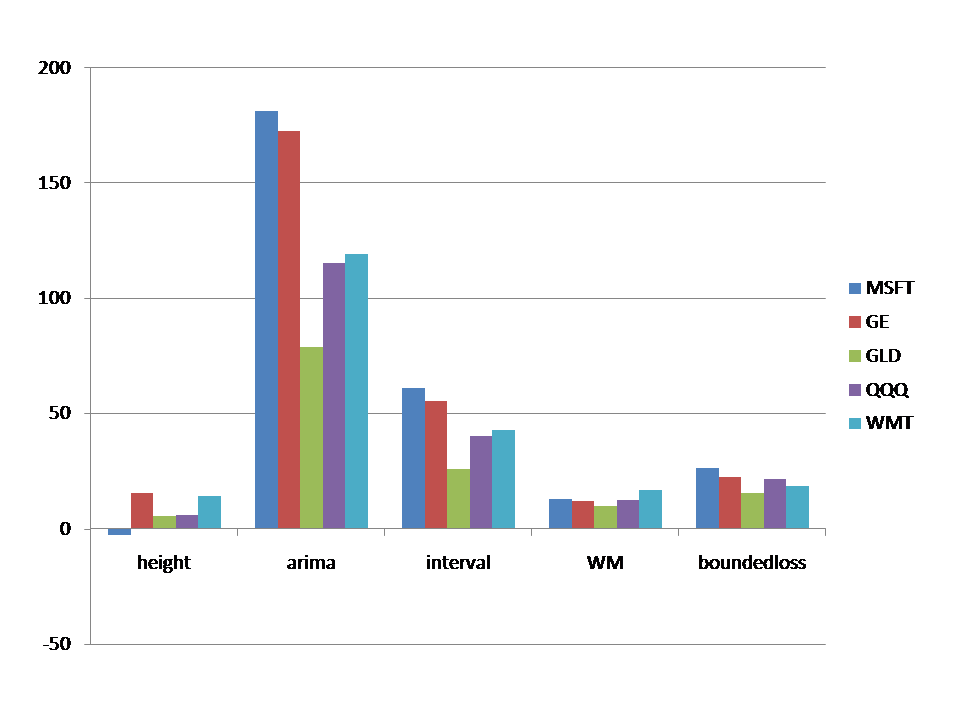}

We note that while our algorithm performs better in practice than other regret minimization based prediction algorithms with provable guarantees, it is outperformed by the ARIMA model.

\section{Omitted Proofs}

For $s$ a sequence of bits of length at most $T$, Let $R(s)$ denote a random string of length $T$ with prefix $s$; that is,  append a random  suffix to $s$ to make it of length $T$. Let $a.b$ denote the concatenation of $a$ and $b$. Let $f(D)$ denote the expected value of $f$ on a string drawn from $D$.  Let $[T] = \{1, 2, \ldots, T\}$.

\begin{observation} \label{impossible}

Let $A$ be an algorithm that guarantees a regret of at most  $c \cdot \sqrt{x}$ on an interval of length $x$ for all sequences.  Then there is a distribution $D$ over sequences of length $kx$ such that the expected
regret of $A$ on $D$ is at least $\Omega(k \cdot \sqrt{x})$.  Setting $k$ to be large enough, this implies that there is no prediction algorithm that can guarantee a regret of $O( \sqrt{|Y|})$ on all intervals $Y$ for all  input sequences.
\end{observation}

\begin{proof}
Let $S_1$ be the sequences of length $x$ with absolute height more than $2c \sqrt{x}$ and $S_2$ be all other sequences of length $x$.  We know that the expected payoff of $A$ on a uniformly random sequence of length $x$ is $0$.  On the other hand, the payoff of $A$ on any sequence in $S_1$ is at least $c \cdot \sqrt{x}$.  A random string of length $x$ falls into $S_1$ with probability $e^{-\Omega(c^2)}$.  Thus, the expected payoff of $A$ on a random string chosen from $S_2$ is at most $-c \sqrt{x} e^{-\Omega(c^2)} = -\Omega(\sqrt{x})$.  

Consider the distribution $D$ over sequences of length $kx$ which is just the concatenation of $k$ random, independent sequences from $S_2$.  Then because $A$ has bounded regret in \emph{every} interval of length $x$, by the same argument as above we would get that the expected payoff of $A$ on $D$ is at most $- \Omega(k \cdot \sqrt{x})$ and hence the expected regret is at least $\Omega(k \cdot \sqrt{x})$.

\end{proof}

\begin{lemma} \label{align-lemma}
If $P_{\alpha}^A$ is feasible then $P_{c \alpha}$ is also feasible, where $c := \frac{\sqrt{2}}{\sqrt{2} - 1}$.
\end{lemma}
\begin{Proof}

Let $X_1, X_2, \ldots, X_k$ or a given sequence $S$.  We split each interval $X_i$ into a disjoint union of aligned intervals $Y_{i1}, \ldots, T_{il}$.  We will then show that the identity 
$$ \sum_{j = 1}^l \sqrt{\abs{Y_{ij}}}  \leq c \cdot \sqrt{\abs{X_i}}$$
always holds where $\abs{I}$ denotes the length of the interval $I$.  This suffices to prove the theorem since $h(X_i) \leq \sum_{j = 1}^l h(Y_{ij})$.

 For notational simplicity, let $I = X_i$ and $x = \abs{I}$.  If $I$ is an aligned interval we are done, otherwise we write it as the minimal union of aligned intervals 
(take out the largest aligned interval in $I$ and repeat).  There are three possibilities:-
\begin{enumerate}
  \item $I = I_1 \cup I_2$ is a union of two intervals of size $x/2$ each (eg. the interval $[T/4 + 1,\ 3T/4]$) \label{first}
  \item $I = I_1 \cup I_2 \cup \ldots \cup I_l$, where each $I_j$ is of a different size.  Note that all interval sizes on the right are powers of $2$ and strictly less than $x$ \label{second}
  \item $I = J \cup J'$ where each $J$ can be written as a union of intervals as in \ref{first} or \ref{second} above
\end{enumerate}

In the first case, 

$$ \sqrt{\abs{I_1}} + \sqrt{\abs{I_2}}  \leq 2 \cdot \sqrt{x/2} = \sqrt{2} \cdot \sqrt{x} $$

In the second case,
$$  \sum_{j = 1}^l \sqrt{\abs{I_j}} \leq \sqrt{x} \cdot \sum_{j=1}^\infty \sqrt{1/2^j} =  \frac{1}{\sqrt{2} - 1}  \cdot \sqrt{x} $$

In the third case, 

$$  \sqrt{\abs{J}} +\sqrt{\abs{J'}} \leq \frac{1}{\sqrt{2} - 1}  \cdot \sqrt{|J|} \ + $$
$$ \frac{1}{\sqrt{2} - 1}  \cdot \sqrt{|J'|} \leq \frac{\sqrt{2}}{\sqrt{2} - 1} \cdot \sqrt{x}$$
\end{Proof}

\subsection{Generalization to values of $b_t$ beyond $[-1,1]$}

In many applications the values $b_t$ may not be bounded in a range such as $[-1,1]$ but could have unbounded values, as in the case when they are drawn from a normal distribution. We will now extend our algorithm to such a case. We will show that our guarantees continue to hold in a semi adversarial setting where an adversary chooses the signs of $b_t$ but its magnitude is chosen from distribution with mean $1$.  Let $D$ denote a distribution over magnitude of real numbers with mean $1$ (and clearly with non-negative support). Let $s$ denote a sequence of bits (as signs) $b_t$. Let $M(s)$ denote a sequence of real numbers where each real number $m_t$ is obtained by multiplying $b_t$ with a randomly and independently drawn value from $D$.

Let $f$ denote a desired payoff function on a sequence of real numbers. We will show a sufficient condition to achieve on a sequence drawn from $M(s)$ an expected payoff of $f(M(s)) = E_{S \in M(s)}[f(S)]$. In the prediction algorithm, instead of appending random bits, we append a numbers with random signs but with magnitudes drawn from $D$. Given a sequence of real numbers $m$. Let $C(m)$ denote a random completion of $m$ to a sequence of length $T$ by appending numbers drawn randomly from $D$ and with a randomly chosen sign ($+1,-1$).

\begin{theorem} \label{real}
Given a payoff function  $f$ defined on a sequence of  real numbers, if $E_s[f(M(s))] \le 0$, then there is a prediction algorithm whose expected payoff on a string drawn from $M(s)$ is at least $f(M(s))$. This is obtained by betting  $\bt_t =  (f(C(m.[+1])) - f(C(m.[-1])))/ 2$, where $s$ is the sequence  seen so far. Note that $\bt_t \in [-1,1]$ as long as  for all  $s$, $|f(C(m.[+1])) - f(C(m.[-1]))| \le 2$
\end{theorem}

\begin{proof}
Let $s$ denote the sequence of signs seen so far. As in Covers proof we can show that setting $\bt_t = (f(M(R(s.[+1]))) - f(M(R(s.[-1]))))/ 2$ ensures that our expected  payoff at time $t$ is at least $f(M(R(s)))$. 

And also note that $(f(C(m.[+1])) - f(C(m.[-1])))/ 2$ in expectation is equal to $(f(M(R(s.[+1]))) - f(M(R(s.[-1]))))/ 2$ as $m$ is distributed as $M(s)$.
\end{proof}

\section{Trade-off with two experts}

Equivalence between the bit-prediction and two experts problem. The following equivalence
is shown in \cite{Andoni-Pmanuscript}. We redo the same proof here for the DP based solution.

In the above formulation we can define loss to be the maximum (-ve)
payoff. and we can obtain a tradeoff between regret $R$ and loss
$L$. This tradeoff is useful in obtaining a tradeoff on two different
regrets when there are two experts. In each round each expert has a
payoff in the range $[0,1]$ that is unknown to the algorithm. For two
experts, let $b_{1t}, b_{2t}$ denote the payoffs of the two
experts. The algorithm pulls the each arm (expert) with probability
$\bt_{1t}, \bt_{2t} \in [0,1]$ respectively where $\bt_{1t} + \bt_{2t}
= 1$.  The payoff of the algorithm is $A = \sum_{t=1}^T b_{1t}\bt_{1t}
+ b_{2t}\bt_{2t} $. Let $X_1 = \sum_{t=1}^T b_{1t}$ We will study the
regret trade-off $R_1, R_2$ with respect to these two experts which
means that $A \ge X_1 - R_1$ and $A \ge X_2 - R_2$.

One question that has been asked before is a tradeoff between regret
to the average and regret to the max \cite{kearns-regret,KP11}.  With
two experts, the regret/loss tradeoff in the sequence prediction
problem is related to regret trade-off for the two experts problem.
Let $R$, $L$ be feasible upper bounds on the regret and loss in the
sequence prediction problem in the worst case; Let $R_o, L_o$ be
feasible upper bounds on the regret and loss with version of the
sequence prediction problem with one sided bets (that is $\bt_t$
cannot be negative; the feasible payoff curves for this case is a
simple variant of $F_{c_1,c_2}$ where $F'$ is capped to lie in
$[0,1]$.) Let $R_1$, $R_2$ be feasible upper bounds in regret with
respect to expert one and expert two in the worst case. Let $R_m$,
$R_a$ be feasible upper bounds on the regret to the max and regret to
the average with two experts in the worst case.

\begin{lemma}[\cite{Andoni-Pmanuscript}]

Then $R,L$ is feasible in the sequence prediction problem if and only if $R_m= R/2, R_a= L/2$ is feasible for regret to the max and regret to the average in the two experts setting.

$R_o,L_o$ is  feasible in the sequence prediction problem  (with one sided bets) if and only if $R_1 = L_o, R_2 = R_o$ is feasible for regret to the first expert and regret to the second expert in the two experts setting.
\end{lemma}

\begin{proof}
First we look at reduction from the regret to the average and regret
to the max problem. We can reduce this problem to our sequence
prediction problem by producing at time $t$, $b_t = (b_{1t} -
b_{2t})/2$. A bet $\tilde b_t$ in our prediction problem can be
translated back probabilities $\bt_{1t} = (1+\bt_t)/2$ and
$(1-\bt_t)/2$ for the two experts. A payoff $A$ in the original
problem gets translated into payoff $\sum_t b_{1t} (1+\bt_t)/2 +
b_{2t} (1-\bt_t)/2 = (X_{1} + X_{2})/2 + A$ in the two experts case.
In this reduction the loss $L$ gets mapped to $R_a$ and the regret $R$
gets mapped to $R_m$. However note that $b_t$ is now in the range
$[0,1/2]$. Therefore we need to scale it by $2$ to reduce it to the
standard version of the original problem. Conversely, given an
sequence $b_t$ of the prediction problem we can convert it into two
experts with payoffs $b_{1t} = (1+b_t)/2, b_{2t} = (1-b_t)/2$. The
average expert has payoff $T/2$.  A payoff of $A$ in prediction
problem can be obtained from a sequence of arm pulling probabilities
with payoff $T/2 + A/2$ by interpreting the arm pulling probabilities
as $(1\pm \bt_t)/2$ since $\sum_t \frac{ (1+b_t)}{2}\frac{
  (1+\bt_t)}{2} + \frac{ (1-b_t)}{2}\frac{ (1-\bt_t)}{2} = T/2 + A/2$.

Next we look at regrets $R_1, R_2$ with respect to the two
experts. Given a sequence of payoffs to for the two experts we can
reduce it to a sequence for the (one sided ) prediction problem by
setting $b_t = b_{2t} - b_{1t}$. A bet $\tilde b_t$ in the prediction
problem can be translated to probabilities $\bt_{1t} = 1-\bt_t$ and
$\bt_{2t} = \bt_t$ for the two experts.  A payoff $A$ in the
prediction problem gets translated into payoff $\sum_t (1-\bt_t)
b_{1t} + \bt_t b_{2t} = X_{1} + A$ in the two experts case where a
zero regret in the prediction would correspond to $A = X_2 -
X_1$. Thus a loss of $L_o$ translates to a regret $R_1 = L_o$ with
respect to the first arm. And regret $R_o$ translates to regret $R_2 =
R_o$ with respect to the second arm. Thus if $R_o, L_o$ is feasible
then so is $R_1 = R_o, R_2 = L_o$. Conversely, given an instance of
the prediction problem with one sided bets, we can convert it to a
version of the two armed problem by setting $b_{2t} = b_t, b_{1t} = 0$
if $b_t \ge 0$ and $b_{2t} = 0, b_{1t} = -b_t$ otherwise.  A bet
$\tilde b_t$ is used in our original problem if the arms are pulled
with probabilities $1-\tilde b_t$ and $\tilde b_t$ respectively. The
payoff in the experts problem is $X_1 + \sum_t \tilde b_t (b_{2t} -
b_{1t})$. So regrets $R_1, R_2$ will translate to $L_o = R_1, R_o =
R_2$ in the prediction problem with one sided bets.
\end{proof}

\end{document}